\documentclass[sigconf]{acmart}
\AtBeginDocument{%
  }
\setcopyright{acmlicensed}
\copyrightyear{2026}
\acmYear{2026}
\setcopyright{cc}
\setcctype{by}
\acmConference[WWW '26]{Proceedings of the ACM Web Conference 2026}{April 13--17, 2026}{Dubai, United Arab Emirates}
\acmBooktitle{Proceedings of the ACM Web Conference 2026 (WWW '26), April 13--17, 2026, Dubai, United Arab Emirates}
\acmPrice{}
\acmDOI{10.1145/3774904.3792704}
\acmISBN{979-8-4007-2307-0/2026/04}
\usepackage{graphicx}
\usepackage{soul}
\usepackage{amsmath}
\usepackage{amsfonts}
\usepackage{amsthm}
\usepackage{textcomp}
\usepackage[linesnumbered,ruled]{algorithm2e}
\usepackage{float}
\usepackage{bm}
\usepackage{multirow}
\makeatletter
\renewcommand{\maketag@@@}[1]{\hbox{\m@th\normalsize\normalfont#1}}%
\makeatother
\usepackage{array} 
\renewcommand\arraystretch{1.3}
\theoremstyle{plain}

\newtheorem{proposition}{Proposition}

\theoremstyle{remark}

\usepackage[most]{tcolorbox}  
\newtcolorbox{problemBox}{
  colback=gray!10,    
  colframe=gray!60,   
  boxrule=0.5pt,      
    coltitle=white,     
  colbacktitle=black,   
  fonttitle=\bfseries, 
  title={Selected Fuzzy Rules}, 
  arc=3pt,            
  left=2pt, right=2pt, top=1pt, bottom=1pt, 
  width=0.48\textwidth, 
  fontupper=\small,   
}
\usepackage{setspace}
\newtcolorbox{problemBox2}{
  colback=gray!10,    
  colframe=gray!60,   
  coltitle=white,     
  colbacktitle=black,   
  fonttitle=\bfseries, 
  title={Node Role-Guided Description}, 
  boxrule=0.5pt,      
  arc=3pt,            
  left=3pt, right=3pt, top=2pt, bottom=2pt, 
  width=0.48\textwidth,
fontupper=\small\setstretch{0.9}   
}
\newtcolorbox{problemBox3}{
  colback=gray!10,    
  colframe=gray!60,   
  coltitle=white,     
  colbacktitle=black,   
  fonttitle=\bfseries, 
  title={Chain-of-Thought Reasoning for Dynamic Graphs}, 
  boxrule=0.5pt,      
  arc=3pt,            
  left=3pt, right=3pt, top=2pt, bottom=2pt, 
  width=0.48\textwidth,
fontupper=\small\setstretch{0.9}   
}
\newtcolorbox{problemBox4}{
  colback=gray!10,    
  colframe=gray!60,   
  coltitle=white,     
  colbacktitle=black,   
  fonttitle=\bfseries, 
  title={Chain-of-Thought Reasoning for Dynamic Graphs}, 
  boxrule=0.5pt,      
  arc=3pt,            
  left=3pt, right=3pt, top=2pt, bottom=2pt, 
  width=0.48\textwidth,
fontupper=\small\setstretch{0.9}   
}

\newtcolorbox{problemBox5}{
  colback=gray!10,    
  colframe=gray!60,   
  coltitle=white,     
  colbacktitle=black,   
  fonttitle=\bfseries, 
  title={Final Explanation Generation}, 
  boxrule=0.5pt,      
  arc=3pt,            
  left=3pt, right=3pt, top=2pt, bottom=2pt, 
  width=0.48\textwidth,
fontupper=\small\setstretch{0.9}   
}

\usepackage{multirow}
\usepackage{graphicx} 

\begin{document}

\title[Node Role-Guided LLMs for Dynamic Graph Clustering
]{Node Role-Guided LLMs for Dynamic Graph Clustering
}
\author{Dongyuan Li}
\affiliation{\institution{The University of Tokyo}
\city{Tokyo}
\country{Japan}}
\orcid{0000-0002-4462-3563}
\email{lidy@csis.u-tokyo.ac.jp}

\author{Ying Zhang}
\affiliation{%
  \institution{RIKEN}
  \city{Sendai}
  \country{Japan}
}
\orcid{0009-0000-9627-8768}
\email{ying.zhang@riken.jp}

\author{Yaozu Wu}
\affiliation{%
  \institution{The University of Tokyo}
  \city{Tokyo}
  \country{Japan}}
\orcid{0009-0005-9766-2186}
\email{yaozuwu279@gmail.com}

\author{Renhe Jiang}
\authornote{Corresponding author.}
\affiliation{%
  \institution{The University of Tokyo}
  \city{Tokyo}
  \country{Japan}}
\orcid{0000-0003-2593-4638}
\email{jiangrh@csis.u-tokyo.ac.jp}

\renewcommand{\shortauthors}{Dongyuan Li, Ying Zhang, Yaozu Wu, and Renhe Jiang}

\begin{abstract}
Dynamic graph clustering aims to detect and track time-varying clusters in dynamic graphs, revealing how complex real-world systems evolve over time. However, existing methods are predominantly black-box models. They lack interpretability in their clustering decisions and fail to provide semantic explanations of why clusters form or how they evolve, severely limiting their use in safety-critical domains such as healthcare or transportation. To address these limitations, we propose an end-to-end interpretable framework that maps continuous graph embeddings into discrete semantic concepts through learnable prototypes. Specifically, we first decompose node representations into orthogonal role and clustering subspaces, so that nodes with similar roles (e.g., hubs, bridges) but different cluster affiliations can be properly distinguished. We then introduce five node role prototypes (Leader, Contributor, Wanderer, Connector, Newcomer) in the role subspace as semantic anchors, transforming continuous embeddings into discrete concepts to facilitate LLM understanding of node roles within communities. Finally, we design a hierarchical LLM reasoning mechanism to generate both clustering results and natural language explanations, while providing consistency feedback as weak supervision to refine node representations. Experimental results on four synthetic and six real-world benchmarks demonstrate the effectiveness, interpretability, and robustness of DyG-RoLLM. Code is available at https: //github.com/Clearloveyuan/DyG-RoLLM.
\end{abstract}

\ccsdesc[500]{Computing methodologies~Temporal reasoning}
\keywords{Graph Clustering, Dynamic Graphs, Temporal Networks, Community Detection, LLM Reasoning, Interpretable Graph Learning.}

\maketitle

\section{Introduction}

Dynamic graph clustering, also known as dynamic community detection, aims to identify and track evolving communities by leveraging both topological structures and temporal dependencies in dynamic graphs~\cite{DBLP:journals/www/RanjkeshMH24,DBLP:conf/ijcnn/LiuWZY19,DBLP:journals/tkde/ChenJLLLCYH24}. 
As an effective tool to reveal the evolutionary mechanisms underlying complex real-world systems, 
dynamic graph clustering has drawn significant attention in various domains, such as social networks analysis~\cite{zhang2022robust,LiYicong,zhong2024efficient,ji2024memmap}, recommendation systems ~\cite{Jiasu,zhang2023dyted,WangWLGLZ22,zhang2023tiger}, and AI4Science~\cite{JiShuo,shi2024graph,sun2024fast}. 

Numerous methods have been proposed for dynamic graph clustering in recent years, which can be broadly classified into three categories~\cite{DBLP:journals/csur/RossettiC18}. \textbf{\textit{(i) Neural network-based methods}}~\cite{DBLP:conf/aaai/YaoJ21,DBLP:journals/tcyb/GaoZZWL23,DBLP:conf/kdd/YouDL22,DBLP:conf/iclr/001400T00024} focus on learning dynamic node representations, with clustering applied as post-processing. 
This two-stage pipeline separates representation learning from clustering, leading to sub-optimal performance~\cite{DBLP:conf/ijcai/DongZSLC18,9531337}. 
\textbf{\textit{(ii) Matrix Factorization-based methods}}~\cite{Chakrabarti,DBLP:journals/tec/MaWWL24,Danon_2005,9531337,LiuF1801,Liu2019} address this limitation by jointly optimizing representation learning and clustering. 
These methods cluster nodes at each timestamp while maintaining temporal smoothness, achieving superior performance on multiple benchmarks~\cite{li2025revisiting,ostroski2025scalable}. 
Recently, \textbf{\textit{(iii) LLM-based methods}}~\cite{ni2024comgpt,gujral2025llms} have emerged to enable semantic reasoning about community structures. 
These approaches integrate GNNs with LLMs through two paradigms: GNN-as-predictor methods~\cite{tang2024graphgpt,zhang2024graphtranslator} enhance community detection by enriching node embeddings with LLM-extracted semantic features, while LLM-as-predictor approaches~\cite{ni2024comgpt,zhou2025few} directly infer cluster assignments by transforming graph topology into natural language prompts. By incorporating language models, these methods identify clusters based on semantic similarity rather than purely structural connections, enabling the detection of functionally coherent clusters that traditional topology-based clustering might miss~\cite{gujral2025llms}.

Despite their success, existing dynamic graph clustering methods share a fundamental limitation: \textbf{\textit{they produce black-box clustering decisions without explanations}}, leaving users unable to understand why certain nodes form clusters or how these clusters evolve over time. This lack of interpretability critically limits their adoption in high-risk applications. For instance, in epidemic control, public health officials need to understand not just which clusters are at risk, but why certain groups form transmission clusters and how intervention strategies might affect cluster evolution~\cite{tran2025fine}.
Although some works use motifs~\cite{chen2023tempme} or causal theory~\cite{zhao2024causality} for explanations, 
such non-textual explanations are abstract, hard to verify, and often fail in complex cases~\cite{lukyanov2025robustness}. 
The emergence of LLMs with strong reasoning capabilities~\cite{chu-etal-2024-navigate} offers a promising direction: using human-understandable natural language to explain clustering decisions. 
However, enabling faithful explanations for graph clustering raises three challenges:
\textbf{\textit{(i) Clustering-aware Representation}}: How can we learn node representations that capture community structure and temporal dynamics while remaining interpretable for downstream explanations? 
\textbf{\textit{(ii) Structure-Semantic Alignment}}: How can we present continuous node representations to LLMs so that structural and temporal properties are preserved for language-based reasoning?
\textbf{\textit{(iii) Reasoning Guidance}}: How can we guide LLMs to generate consistent, verifiable explanations aligned with the graph rather than hallucinating plausible but incorrect rationales?

To answer these questions, we propose DyG-RoLLM{\setlength{\skip\footins}{0.5pt}\setlength{\footnotesep}{0.25\baselineskip}\footnote{DyG-RoLLM denotes \textbf{\underline{Dy}}namic \textbf{\underline{G}}raph Clustering via Node \textbf{\underline{Ro}}le-guided \textbf{\underline{LLM}}.}}, the first framework to provide interpretable dynamic graph clustering via bidirectional GNN-LLM interaction.
\textbf{\textit{First}}, to learn clustering-aware representation, we simultaneously learn clustering patterns and interpretable features through orthogonal subspace decomposition. GAT-GRU\footnote{Graph Attention Network (GAT) encoder with a Gated Recurrent Unit (GRU) updater.} captures structural-temporal dynamics, whose outputs are then separated into two orthogonal subspaces: one optimizes modularity for community detection, while the other learns features for semantic interpretation. This orthogonal design prevents the two objectives from interfering with each other.
\textbf{\textit{Second}}, to establish structure-semantic alignment, the interpretable subspace specifically learns five role prototypes (Leader, Contributor, Wanderer, Connector, Newcomer) that bridge continuous embeddings to discrete language. These prototypes transform node representations into role distributions (e.g., node $v_{24}$ has 70\% Leader traits), which generate hierarchical descriptions spanning node-level roles, community-level compositions, and evolution-level patterns. This multi-level semantic representation provides LLMs with comprehensive context for understanding each node's position, function, and trajectory within the dynamic graph.
\textbf{\textit{Finally}}, to provide reasoning guidance, we leverage these hierarchical descriptions in a chain-of-thought framework where LLMs analyze role-community compatibility, supply-demand imbalances, and temporal consistency to determine cluster assignments. 
Crucially, high-confidence LLM predictions serve as soft supervision for GNN training through consistency loss, while GNN provides structural context for LLM reasoning. This bidirectional loop enables semantic understanding to refine structural clustering and structural patterns to guide semantic interpretation, yielding both accurate clustering and faithful explanations. The main contributions can be summarized as:
\begin{itemize}
\item We formalize the interpretable dynamic graph clustering problem, and propose the first framework that generates natural language explanations alongside clustering decisions.
\item We introduce role prototypes as semantic anchors that bridge neural and symbolic spaces, enabling the derivation of community and evolution interpretations from node-level roles.
\item We establish a GNN-LLM interaction with CoT reasoning for dynamic graphs, validated via case studies showing how LLM reasoning corrects structural misclustering.
\item Experimental results on ten benchmarks demonstrate that DyG-RoLLM achieves state-of-the-art clustering performance while providing human-interpretable explanations.
\end{itemize}

\section{Related Work}
\label{related}

Graph Neural Networks have gradually become the dominant paradigm for training graph models, and a wealth of outstanding works has emerged accordingly~\cite{xia2025graph,yu2025cage,yu2025graph2text,yu2024formulating,yu2025path,zheng2025dp,chen2024towards,yang2025okg,kong2024traffexplainer}.
Current dynamic graph clustering methods can be mainly classified into three categories~\cite{9388900,LiLM21,Dyg-Mamba}: Neural Network (NN)-based methods, Matrix Factorization (MF)-based methods, and LLM-based methods. 
\textit{\textbf{First}}, NN-based methods often employ a coupled strategy, where they first condense dynamic graphs into a single static graph and then apply clustering methods to identify clusters~\cite{9314087, DBLP:journals/kbs/SantoGMS21,DBLP:journals/pami/ZhangNL23,zheng2022instant}. Other NN-based methods employ two-stage strategies, where they first learn dynamic graph embeddings and then apply clustering methods on these embeddings to identify communities~\cite{beer2023connecting,DBLP:conf/kdd/ZhangCFXZSC23,guo2022subset,zhao2023spatial,DBLP:conf/KDD/CrossCityTransfer22,DBLP:journals/tnn/CuiLWZLWA24,Namyong,yang2024effective,chen2022efficient}. 
For example, RNNGCN~\cite{DBLP:conf/aaai/YaoJ21} and DGCN~\cite{DBLP:journals/tcyb/GaoZZWL23} use RNNs or LSTMs to capture temporal dependencies for graph embeddings, which are then clustered using graph convolutional layers. 
ROLAND~\cite{DBLP:conf/kdd/YouDL22} extends static GNN-based graph embedding methods to dynamic graphs by using GRUs to capture temporal information. 
\textit{\textbf{Second}}, MF-based methods perform matrix factorization at each timestamp to cluster nodes while enforcing temporal consistency across snapshots. 
These methods differ primarily in how they measure and maintain temporal smoothness. One category exploits topology changes: sE-NMF~\cite{Ma17b}, jLMDC~\cite{9531337}, and NE2NMF~\cite{DBLP:journals/kbs/LiZDGM21} analyze changes between consecutive graphs, while PisCES~\cite{LiuF1801} considers global topology evolution. 
Another category directly optimizes clustering metrics: DynaMo~\cite{DBLP:journals/tkde/ZhuangCL21} incrementally maximizes modularity, and MODPSO~\cite{DBLP:journals/isci/YinZLD21} minimizes cross-timestamp NMI. 
\textit{\textbf{Third}}, LLM-based methods integrate GNNs with LLMs for graph understanding, following two primary paradigms depending on which component serves as the final predictor. 
\textit{GNN-as-predictor} approaches~\cite{tang2024higpt,ma2024xrec,he2024g,zhao2023gimlet,tian2024graph,he2024unigraph} enhance GNN inputs through LLM-processed features. 
For example, GraphGPT~\cite{tang2024graphgpt} enhances graph prediction by combining semantic embeddings from LLMs with structural embeddings from GNNs.  
GraphTranslator~\cite{zhang2024graphtranslator} employs LLMs to extract semantic features from text descriptions of graphs. 
UniGraph~\cite{he2024unigraph} uses pre-trained LLMs to encode textual node attributes. 
\textit{LLM-as-predictor} methods~\cite{xie2023graph,wei2024llmrec,xia2024opengraph,guo2024graphedit,ren2024representation} transform graph structure into natural language for direct LLM reasoning. For example, ComGPT~\cite{ni2024comgpt}  encodes adjacency information with iterative prompting, STC-CDP~\cite{zhou2025few} combines structural descriptions with few-shot demonstrations, and CommLLM~\cite{gujral2025llms} converts entire graphs into text narratives. 
\textbf{\textit{In contrast to all existing methods}}, DyG-RoLLM fundamentally reimagines dynamic graph clustering through three key innovations: \textit{(i)} While existing methods produce black-box results, we generate natural language explanations for every clustering decision; \textit{(ii)} Unlike unidirectional GNN-LLM pipelines, we establish bidirectional interaction where LLMs provide soft supervision for GNN training while GNNs guide LLM reasoning; \textit{(iii)} Beyond simple feature enhancement, our learnable role prototypes create semantic bridges that enable reasoning about why communities form and how they evolve. 

\section{Preliminary}
\label{preliminaries}

\noindent \textbf{Dynamic Graph Clustering.} 
We consider a dynamic graph as a sequence of snapshots $\mathcal{G} = \{G_1, G_2, \ldots, G_\tau\}$ and the $t$-th snapshot $G_{t}=(V^{t}, E^{t}$, $X^{t}$), defined for $0\leq t\leq \tau$. Here, ${V}^{t}$, ${E}^{t}$ and ${X}^{t}$ represent the set of nodes, edges, and node features at $t$-th snapshot. 
Dynamic community detection aims to partition the node set $V^t$ into non-overlapping clusters \begin{small}$\mathbf{C}^t = \{\mathbf{C}_1^t, \ldots, \mathbf{C}_{K_t}^t\}$\end{small} where $K_t$ denotes the number of clusters at $t$-th timestamp, \begin{small}$\bigcup_{k=1}^{K_t} \mathbf{C}_k^t = V^t$\end{small}, and \begin{small}$\mathbf{C}_i^t \cap \mathbf{C}_j^t = \emptyset$\end{small} for $i \neq j$.
Beyond traditional clustering, our task learns the mapping 
\begin{small}$\Psi: G^t \rightarrow (\mathcal{C}^t, \{\mathcal{E}_i\}_{i=1}^{|V^t|})$\end{small}, where $\mathcal{E}_i$ represents hierarchical natural language explanations for node $i$'s community assignment.

\begin{figure*}[t]
\centering
\includegraphics*[clip=true,width=\textwidth]{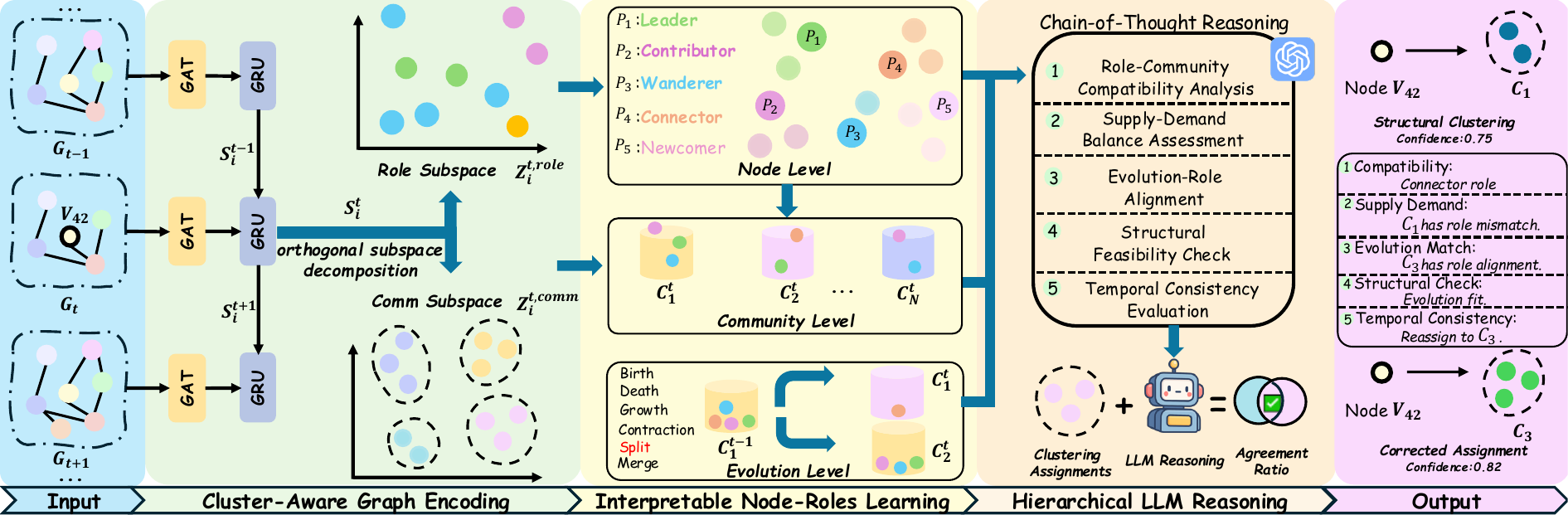}
\caption{Overview of DyG-RoLLM. We first learn cluster-aware representation through orthogonal decomposition, then bridge them to semantics via node-roles prototypes, and finally enable hierarchical LLM reasoning for interpretable clustering.}\label{figure:illustrate_smftii2}
\end{figure*}

\section{Methodology}
\label{methodology}

As shown in Figure~\ref{figure:illustrate_smftii2}, DyG-RoLLM consists of three components: \textbf{(i) Clustering-aware Graph Encoding} employs orthogonal subspace decomposition to separate role features from community features ({Sec.\ref{sec-M1}}), \textbf{(ii) Interpretable Node-Roles Learning} defines semantic prototypes bridging graph embeddings with natural language ({Sec.\ref{sec-M2}}), and \textbf{(iii) Hierarchical LLM Reasoning} performs structured CoT for interpretable clustering ({Sec.\ref{sec-M3}}).

\subsection{Clustering-Aware Graph Encoding}
\label{sec-M1}

To capture evolving clustering patterns in dynamic graphs, we design a dynamic graph encoding module with \textbf{orthogonal subspace decomposition}. This module learns node representations that jointly encode node roles and clusters in independent subspaces, providing the foundation for interpretable graph clustering.

\subsubsection{Dynamic Graph Encoding.}
Given a dynamic graph $\mathcal{G}$, we first encode the structural information at each timestep $G_t=\{V^t,E^t\}$ using a multi-layer GAT. 
For a node $i$ at time $t$ with feature vector $\mathbf{x}_i^t \in \mathbb{R}^{d}$, the $l$-th GAT layer computes the node representation as:
\begin{footnotesize}
\begin{equation}
\mathbf{h}_i^{t,(l)} = \text{ReLU}(\sum_{j \in \mathcal{N}_i^{t}} \alpha_{ij}^{t,(l)} \mathbf{W}^{(l)} \mathbf{h}_j^{t,(l-1)})
\end{equation}
\end{footnotesize}where $\mathbf{h}_i^{t,(0)} = \mathbf{x}_i^t$, $\mathcal{N}_i^{t}$ is the neighborhood of node $i$ at timestep $t$, \begin{small}$\mathbf{W}^{(l)}$\end{small} is the trainable weight matrix, and \begin{small}$\alpha_{ij}^{t,(l)}$\end{small} can be computed as:
\begin{footnotesize}
\begin{equation}
\alpha_{ij}^{t,(l)} = \frac{\exp\left(\text{LeakyReLU}\left(\mathbf{a}^{\top}[\mathbf{W}^{(l)} \mathbf{h}_i^{t,(l-1)} \| \mathbf{W}^{(l)} \mathbf{h}_j^{t,(l-1)}]\right)\right)}{\sum_{m \in \mathcal{N}_i^t \cup \{i\}} \exp\left(\text{LeakyReLU}\left(\mathbf{a}^{\top}[\mathbf{W}^{(l)} \mathbf{h}_i^{t,(l-1)} \| \mathbf{W}^{(l)} \mathbf{h}_m^{t,(l-1)}]\right)\right)},
\end{equation}
\end{footnotesize}

\noindent where $\|$ denotes concatenation and $\mathbf{a} \in \mathbb{R}^{2d}$ is a learnable vector. 
To capture temporal dependencies across snapshots, we employ a Gated Recurrent Unit (GRU) that maintains node states over time:
\begin{footnotesize}
\begin{equation}
\mathbf{s}_i^{t} = \text{GRU}\left(\mathbf{h}_i^{t,(L)}, \mathbf{s}_i^{t-1}\right)
\end{equation}
\end{footnotesize}where $\mathbf{h}_i^{t,(L)}$ is the final output, and $\mathbf{s}_i^{t} \in \mathbb{R}^{d}$ represents the node embedding at timestep $t$. A fundamental conflict arises: nodes with similar roles should cluster together, yet may belong to different communities requiring separation. We solve this via \textit{orthogonal subspace decomposition}, as guaranteed by the following proposition.

\begin{proposition}[Gradient Isolation via Orthogonal Decomposition]
\label{thm:gradient_isolation}
Given a decomposition function $\mathcal{F}:\mathbb{R}^{d}\to\mathbb{R}^{d_1}\times\mathbb{R}^{d_2}$ defined as $\mathcal{F}(\mathbf{x}) = [\mathbf{W}_1 \mathbf{x}\,\|\, \mathbf{W}_2 \mathbf{x}]$ where $\mathbf{W}_1$ and $\mathbf{W}_2$ are projection matrices satisfying the orthogonality constraint $\mathbf{W}_1 \mathbf{W}_2^{\top} = \mathbf{0}$. For a composite loss $\mathcal{L} = \mathcal{L}_1(\mathbf{W}_1\mathbf{x}) + \mathcal{L}_2(\mathbf{W}_2\mathbf{x})$, the following gradient isolation property holds: $\nabla_{\mathbf{W}_1} \mathcal{L}_1 = \mathbf{0}$ and $\nabla_{\mathbf{W}_2} \mathcal{L}_2 = \mathbf{0}$. This ensures that optimizing one objective does not interfere with the other, allowing both tasks to achieve their respective optima without compromise.
\end{proposition}

Motivated by Proposition~\ref{thm:gradient_isolation} (Appendix~\ref{app:proof_gradient_isolation}), we decompose node representations $\mathbf{s}_i^t$ into orthogonal role and clustering subspaces as: 
\begin{align}
\mathbf{z}_i^t = [\mathbf{z}_i^{t,\text{role}} \,\, \| \, \,\mathbf{z}_i^{t,\text{comm}}] &= [\mathbf{W}_{\text{role}} \cdot \mathbf{s}_i^t\,\|\, \mathbf{W}_{\text{comm}} \cdot \mathbf{s}_i^t]\\
s.t.,\,\, \mathbf{W}_{\text{role}} \,&\mathbf{W}_{\text{comm}}^{\top}  = \mathbf{0}_{d_r \times d_c}\nonumber
\end{align}
where $\mathbf{z}_i^{t,\text{role}} \in \mathbb{R}^{d_r}$ encodes role features, $\mathbf{z}_i^{t,\text{comm}} \in \mathbb{R}^{d_c}$ captures cluster affiliation, and $\mathbf{W}_{\text{role}} \in \mathbb{R}^{d_r \times d}$ and $\mathbf{W}_{\text{comm}} \in \mathbb{R}^{d_c \times d}$ are two learnable orthogonal projection matrix.

\subsubsection{Modularity-Guided Clustering.}\label{CommLoss} With decomposed orthogonal representations, we now optimize node representations in the cluster subspace to obtain clustering-aware node features.

We first generate a soft cluster assignment matrix $\mathbf{C}^t \in \mathbb{R}^{|V^t| \times K_t}$ from the cluster subspace, where $\mathbf{C}_{ij}^t$ represents the probability that the $i$-th node belongs to the $j$-th cluster, which is formulated as:
\begin{equation}
\mathbf{C}_{ij}^t = \frac{\exp((\mathbf{z}_i^{t,\text{comm}})^{\top} \boldsymbol{\mu}_j^{\text{t}}/\tau_c)}
{\sum_{\ell=1}^{K_{t}} \exp((\mathbf{z}_i^{t,\text{comm}})^{\top} \boldsymbol{\mu}_{\ell}^{\text{t}}/\tau_c)},
\end{equation}
where $\boldsymbol{\mu}_\ell^{t} \in \mathbb{R}^{d_c}$ denotes the $\ell$-th learnable cluster centroid, and $\tau_c$ is a temperature parameter that controls the sharpness of assignment.

\vspace{3pt}
\noindent \textbf{Modularity-based Clustering Loss.}
Previous works~\cite{KDD-24,9531337} directly maximize the modularity objective \begin{small}$\mathcal{Q}_\text{modul} = \mathrm{Tr}(\mathbf{C}^{\top} \mathbf{B} \mathbf{C})$\end{small}, where \begin{small}$\mathbf{B}= \mathbf{A} - \frac{\mathbf{d}\mathbf{d}^{\top}}{2|E^t|}$\end{small} is the modularity matrix with \begin{small}$\mathbf{d}_i=\sum_j\mathbf{A}_{ij}$\end{small}, 
to encourage densely connected nodes to cluster together while separating sparsely connected ones. 
However, this formulation requires \begin{small}$\mathcal{O}(|V^t|^2)$\end{small} complexity, severely limiting the scalability to large graphs. To address this computational bottleneck, we reformulate modularity optimization through an efficient decomposition. By leveraging the algebraic structure of the modularity matrix, we derive an equivalent objective that avoids quadratic complexity as:
\begin{small}
\begin{equation}\label{clustering-loss}
\mathcal{L}_{\text{mod}}^t = -\sum_{(i,j) \in E^t} \text{sim}(\mathbf{z}_i^{t,\text{comm}}, \mathbf{z}_j^{t,\text{comm}}) 
+ \frac{1}{2|E^t|} \left\|\sum_{i \in V^t} \mathbf{d}_i \mathbf{z}_i^{t,\text{comm}}\right\|^2,
\end{equation}
\end{small}where the first term aggregates similarities over existing edges, encouraging connected nodes to have similar embeddings, and the second term acts as a degree-aware regularizer that prevents high-degree nodes from dominating the embedding space. 
This decomposition reduces the complexity from \begin{small}$\mathcal{O}(|V^t|^2)$\end{small} to \begin{small}$\mathcal{O}(|E^t| + |V^t|d_c)$\end{small}, enabling efficient optimization on web-scale graphs while preserving the theoretical properties of modularity maximization.

\noindent \textbf{Temporal Smoothing Loss.} 
To ensure smooth cluster evolution while adapting to structural changes, we regularize assignment consistency for persistent nodes between successive snapshots as:
\begin{equation}\label{smoothing-loss}
\mathcal{L}_{\text{temp}}^t = \frac{1}{|V^t \cap V^{t-1}|} \sum_{i \in V^t \cap V^{t-1}} 
\|\mathbf{C}_{i:}^t - \mathbf{C}_{i:}^{t-1}\|^2 \cdot \exp(-\alpha \cdot \delta_i)
\end{equation}
where \begin{small}$\mathbf{C}_{i:}^t$\end{small} is node $i$'s cluster assignment,  \begin{small}$\delta_i= \frac{\|\mathbf{A}_{i:}^t - \mathbf{A}_{i:}^{t-1}\|_1}{\|\mathbf{A}_{i:}^t\|_1 + \|\mathbf{A}_{i:}^{t-1}\|_1}$\end{small} measures local structural change and $\alpha$ controls adaptation sensitivity. This formulation allows for larger assignment changes for nodes with significant neighborhood modifications while maintaining stability.

\noindent \textbf{Overall Clustering Objective.}
The complete loss function for cluster-aware representation learning can be formulated as:
\begin{equation}
\mathcal{L}_{\text{cluster}} = \mathcal{L}_{\text{mod}}^t + \lambda_1 \mathcal{L}_{\text{temp}}^t \quad s.t., \,\,\mathbf{W}_{\text{role}}\mathbf{W}_{\text{comm}}^{\top}  = \mathbf{0}_{d_r \times d_c}
\end{equation}
where $\lambda_1$ balances the trade-off between modularity maximization and temporal smoothness. This joint optimization ensures that learned clusters are topologically meaningful and temporally stable.

\subsection{Interpretable Node-Roles Learning}
\label{sec-M2}

This module establishes the node-role prototypes as semantic anchors, bridging the continuous graph embedding space and the discrete language space for LLM-driven inference. 
The prototypes are designed to be dynamic, robust, and intrinsically interpretable across node-level, community-level, and evolution-level scales.

\subsubsection{Node Role Definition} With orthogonally decomposed representations that separate role from community features, we focus on the role subspace to define interpretable node-role prototypes that enable semantic understanding of clustering results. 

\noindent \textbf{Definition\,\,1} \textbf{(Node Roles)}. 
Drawing from \textit{Structural Node Role Theory}~\cite{borgatti2000models,newman2018networks}, we define five node prototypes $\mathbf{P}^{\text{node}} = \{\mathbf{p}_1, ..., \mathbf{p}_{5}\}$ serving as functional roles with natural-language interpretation as:
\begin{itemize}
\item $\mathbf{p}_1$: {[Leader]} - High centrality nodes, {influential hubs} maintaining a \textit{stable core set of internal connections}.
\item $\mathbf{p}_2$: {[Contributor]} - Moderate connectivity nodes, consistently present, primarily forming \textit{intra-community connections}.
\item $\mathbf{p}_3$: {[Wanderer]} - Weakly connected nodes at boundaries with \textit{low clustering coefficient} and \textit{high migration potential}.
\item $\mathbf{p}_4$: {[Connector]} - Inter-community connectors with \textit{high betweenness centrality} and \textit{low internal clustering coefficient}.
\item $\mathbf{p}_5$: {[Newcomer]} - Recently joined nodes with \textit{high connection volatility} and \textit{rapid changes in structural metrics}.
\end{itemize}

\noindent \textbf{Node Roles Initialization.}
To ensure meaningful role semantics, 
we initialize the $i$-th prototype as $\mathbf{f}^{\text{prior}}_{i}=[\mathbf{f}^\mathbf{C}_{i};\mathbf{f}^\mathbf{L}_{i};\mathbf{f}^\mathbf{T}_{i}]$, where the centrality features $\mathbf{f}^\mathbf{C}$ include degree, betweenness, and closeness centrality; 
local structure features $\mathbf{f}^{\mathbf{L}}$ include the clustering coefficient and intra-community density; 
temporal volatility features $\mathbf{f}^{\mathbf{T}}$ include connection volatility and role stability. 
These synthetic vectors encode ideal role features (details in Appendix~\ref{AppendB-initialization}). 
We then use a shared Multi-Layer Perceptron ($\text{MLP}_{\theta_1}$) to map the role-specific synthetic feature prior to the role representation subspace as:
\begin{equation}
\mathbf{p}_i = \text{MLP}_{\theta_1}\left( \mathbf{f}^{\text{prior}}_{i} \right) + \boldsymbol{\epsilon}_i, \quad \text{where} \,\, \boldsymbol{\epsilon}_i \sim \mathcal{N}(0, \sigma^2I).
\end{equation}

\subsubsection{Role-Embedding Alignment} To align node representations with the defined role prototypes, we employ contrastive learning in the role subspace. This mechanism enables unsupervised discovery of node roles without requiring manual annotations. Specifically, for each node $i$ with role representation $\mathbf{z}_i^{t,\text{role}} \in \mathbb{R}^{d_r}$, we calculate its affinity to each node-role prototype as follows:
\begin{equation}
\pi_{ik}^t = \frac{\exp(\text{sim}(\mathbf{z}_i^{t,\text{role}}, \mathbf{p}_k)/\tau_r)}{\sum_{j=1}^{5} \exp(\text{sim}(\mathbf{z}_i^{t,\text{role}}, \mathbf{p}_j)/\tau_r)}
\end{equation}
where $\text{sim}(\cdot,\cdot)$ denotes cosine similarity and $\tau_r$ is the temperature parameter. This soft assignment allows nodes to exhibit mixed roles, i.e., a node might be 70\% Leader and 30\% Connector, simultaneously.

\noindent \textbf{Contrastive Prototype Loss.} 
We optimize prototypes and node representations jointly via a contrastive objective that attracts nodes to their dominant prototype while maintaining prototype diversity:
\begin{small}
\begin{equation}\label{contrastive-loss}
\mathcal{L}_{\text{proto}}^t = -\frac{1}{|V^t|}\sum_{v_i \in V^t} \log \pi_{ik^*}^t + \lambda_{\text{2}} \sum_{k=1}^5 \hat{\pi}_k^t \log \hat{\pi}_{k}^t,
\end{equation}    
\end{small}where $k^* = \arg\max_k \pi_{ik}^t$ and \begin{small}$\hat{\pi}_k^t = \frac{1}{|V^t|}\sum_{v_i \in V^t} \pi_{ik}^t$\end{small}. The first term attracts nodes toward their nearest prototype, while the second term prevents prototype collapse and ensures all roles remain active.

\subsubsection{Role-Semantic Alignment} The learned role prototypes bridge neural representations and natural language by generating interpretable descriptions at multiple granularities—from individual nodes to community compositions to temporal evolution patterns.

\noindent \textbf{Community Composition.} For each detected community \begin{small}$\{\mathbf{C}_k^t\}_{k=1}^{K_t}$\end{small}, we derive its compositional signature by aggregating member roles: 
\begin{small}
\begin{equation}
\boldsymbol{\gamma}_{\mathbf{C}_k}^t = \frac{1}{|\mathbf{C}_k^t|} \sum_{v_i \in \mathbf{C}_k^t} \boldsymbol{\pi}_i^t = [\gamma_{\mathbf{C}_k,1}^t, ..., \gamma_{\mathbf{C}_k,5}^t]
\end{equation}
\end{small}where $\boldsymbol{\pi}_i^t = [\pi_{i1}^t, ..., \pi_{i5}^t]$ is node $i$'s role distribution. This composition $\boldsymbol{\gamma}^t$ directly reveals community features for LLM interpretation.

\noindent \textbf{Evolution Pattern.}
Temporal changes in role distributions naturally map to evolution patterns. Given two snapshots, we compute:
\begin{small}
\begin{equation}
\Delta\boldsymbol{\gamma}_{\mathbf{C}_k}^t = \boldsymbol{\gamma}_{\mathbf{C}^{t+1}_k} - \boldsymbol{\gamma}_{\mathbf{C}^t_k}
\end{equation}
\end{small}to help LLMs understand six evolution events, including \textit{Birth, Death, Growth, Contraction, Split, Merge} (details in Appendix~\ref{appendix:events}). 

\noindent \textbf{Natural Language Template Generation.}
Each node can be converted to a semantic description by a parametric template function $\mathcal{T}(\cdot)$, which is listed in {Appendix~\ref{append-prompt1}}. For node $v_i$, we can generate the following semantic description via $\mathcal{T}(\boldsymbol{\pi}_{i}^t,{\boldsymbol{\gamma}^t_{\mathbf{C}}}, {\Delta\boldsymbol{\gamma}^t_{\mathbf{C}}},\mathcal{G},\mathbf{C}, \Theta)$, where $\mathcal{G}$ denotes dynamic graphs and $\Theta$ is the activation threshold:

\begin{center}
\begin{problemBox2}
\footnotesize
{\textbf{Node-level:}} \textit{Node $v_{23}$ exhibits dominant Leader characteristics (70\%) with secondary Contributor traits (20\%). It maintains 45 connections with betweenness centrality 0.82, serving as a critical hub in the network structure.}

{\textbf{Community-level:}} \textit{Node $v_{23}$ belongs to Community $\mathbf{C}_3$ (89 members) with composition [15\% Leaders, 50\% Contributors, 25\% Wanderers, 8\% Connectors, 2\% Newcomers]. As a primary leader in this hierarchical core community, $v_{23}$ coordinates with 12 leaders to maintain structural cohesion.}

{\textbf{Evolution-level:}} \textit{Over the past timestep, the community $v_{23}$' experienced a growth pattern with 18\% newcomer influx while maintaining stable leadership. Node $v_{23}$'s role evolved from 55\% to 70\% Leader confidence, indicating consolidation of its leadership position during expansion phase.}
\end{problemBox2}
\end{center}

\subsection{Hierarchical LLM Reasoning}
\label{sec-M3}

We employ LLMs to perform structured reasoning about community assignments, ensuring that detected communities are not only structurally coherent but also logically meaningful.

\subsubsection{Structured Reasoning} Given the hierarchical semantic descriptions generated by function $\mathcal{T}(\cdot)$, LLMs perform cluster assignment via structured probabilistic reasoning:
\begin{small}
\begin{equation}
\mathbf{Q}^t = \text{LLM}_\text{Reason}\left(\mathcal{T}(\{\boldsymbol{\pi}_i^t\}_{i \in V^t}, \{\boldsymbol{\gamma}_{\mathcal{C}_k}^t,\Delta\boldsymbol{\gamma}_{\mathbf{C}_k}^t\}_{k=1}^{K_t}, \mathcal{G}, \mathbf{C}, \Theta)\right),
\end{equation}
\end{small}where \begin{small}$\mathbf{Q}^t$\end{small} denotes the assignment probabilities. We design a reasoning chain to determine optimal cluster assignments as:
\begin{center}
\begin{problemBox3}
\footnotesize
{\textbf{INPUT}}: Three-level semantic description from template $\mathcal{T}(\cdot)$\\[3pt]
{\textbf{REASONING STEPS}}:\\
\textbf{Step 1 - Role-Community Compatibility Analysis}\\
\textit{Evaluate how node's role distribution aligns with each community's composition}\\
\textbf{Step 2 - Supply-Demand Balance Assessment}\\
\textit{Identify imbalances in role distributions across communities}\\
\textbf{Step 3 - Evolution-Role Alignment}\\
\textit{Match node's role trajectory with community evolution patterns}\\
\textbf{Step 4 - Structural Feasibility Check}\\
\textit{Validate assignment with connectivity patterns}\\
\textbf{Step 5 - Temporal Consistency Evaluation}\\
\textit{Consider historical assignments and migration costs}\\[3pt]
{\textbf{OUTPUT}}: Probability distribution $\mathbf{Q}_{i,k}$ over $K$ communities with justification
\end{problemBox3}
\end{center}

\subsubsection{Consistency Learning}

To align structural clustering and semantic reasoning, we quantify the agreement between clustering assignments $\mathbf{C}^t$ and LLM reasoning $\mathbf{Q}^t$ using:
\begin{equation}
\mathcal{A}^t = \frac{1}{|V^t|} \sum_{i \in V^t} \mathbb{I}[\arg\max_k \mathbf{C}_{ik}^t = \arg\max_k \mathbf{Q}_{ik}^t] \in [0,1],
\end{equation}
where $\mathcal{A}$ is the agreement ratio. Low agreement indicates structural patterns contradict semantic logic. Rather than using all LLM outputs, we selectively use high-confidence semantic insights as:
\begin{equation}\label{consistant-loss}
\mathcal{L}_{\text{consist}} = -\frac{1}{|\mathcal{V}_{\text{conf}}^t|}\sum_{i \in \mathcal{V}_{\text{conf}}^t} \sum_{k=1}^{K_t} \mathbf{Q}_{ik}^t \log \mathcal{C}_{ik}^t
\end{equation}
where $\mathcal{V}_{\text{conf}}^t = \{i : \max_k \mathbf{Q}_{ik}^t > \tau_{\text{conf}}\}$ contains nodes with confident LLM predictions (typically $\tau_{\text{conf}} = 0.8$). This prevents uncertain reasoning from destabilizing structural optimization.

\noindent \textbf{Objective Function}. The total optimization objective becomes:
\begin{equation}
\mathcal{L}_{\text{total}} = \mathcal{L}_{\text{cluster}} + \mathcal{L}_{\text{proto}} + \lambda_{\text{LLM}}(\mathcal{A}^t) \mathcal{L}_{\text{consist}}
\end{equation}
where the adaptive weight $\lambda_{\text{LLM}}(\mathcal{A}^t) = \max(0, 1 - \mathcal{A}^t)$ increases LLM influence when disagreement is high.

\subsubsection{Interpretable Outputting} The framework produces both final assignments and comprehensive explanations.

\noindent \textbf{Final Assignment.}
We combine $\mathbf{C}^t$ with LLM reasoning $\mathbf{Q}^t$ as:
\begin{equation}
\hat{k}_i = \arg\max_k [\alpha \mathbf{C}_{ik}^t + (1-\alpha) \mathbf{Q}_{ik}^t]
\end{equation}
where $\alpha=0.7$ balances structural coherence and  semantic insights.

\noindent \textbf{Explanation Generation.}
Each assignment produces explanations aligned with the five-step reasoning chain:
\begin{center}
\begin{problemBox5}
\footnotesize
{\textbf{Assignment Decision:}} \textbf{Node $v_{23}$ → Community $\mathbf{C}_2$}\\[3pt]
{\textbf{Step 1 - Compatibility}}: How node's roles fit community composition.\\
\textbf{Example}: Leader (70\%) + Connector (20\%) complements $\mathbf{C}_2$'s role gaps.\\
{\textbf{Step 2 - Supply-Demand}}: Role distribution imbalances across communities.\\
\textbf{Example}: $\mathbf{C}_1$ oversupplied (35\% Leaders) vs $\mathbf{C}_2$ undersupplied (5\% Leaders).\\
{\textbf{Step 3 - Evolution Match}}: Node trajectory alignment with community dynamics.\\
\textbf{Example}: Rising Leader confidence (55\%→70\%) matches $\mathbf{C}_2$'s growth needs.\\
{\textbf{Step 4 - Structural Check}}: Connectivity validation for assignment.\\
\textbf{Example}: 12 of 45 connections in $\mathbf{C}_2$; betweenness 0.82 supports transition.\\
{\textbf{Step 5 - Temporal Consistency}}: Historical assignment and migration feasibility.\\
\textbf{Example}:  Connector trait (20\%) enables smooth migration from $\mathbf{C}_1$ to $\mathbf{C}_2$.\\[3pt]
{\textbf{Final Confidence}}: 0.81 (Structural: 0.75, Semantic: 0.88).
\end{problemBox5}
\end{center}

\subsection{Complexity Analysis}

\noindent \textbf{Time Complexity.}
The overall time complexity per snapshot is $\mathcal{O}(L|E^t|d + |V^t|K_td_c + |V^t|d_r)$, where $L$ is the number of GAT layers, $d,d_c,d_r$ is the dimensionality of the embedding space, cluster space, role space, respectively. $K_t$ is number of clusters in the $t$-th timestep.
\noindent \textbf{Space Complexity.}
The space requirement is $\mathcal{O}(|V^t|d + K_td_c)$. For web-scale deployment, embeddings can be computed in mini-batches, requiring only $\mathcal{O}(bd)$ memory where $b$ is batch size. We provide more details about complexity analysis in 
Appendix~\ref{append-complexity}.

\renewcommand\arraystretch{1}
\begin{table}[h]
\footnotesize
\centering 
\caption{Detailed statistics of dynamic graph benchmarks.}
\setlength{\tabcolsep}{3mm}{
\begin{tabular}{@{}lccc@{}}
\toprule 
\textbf{Dynamic Graphs} & \textbf{\# of Nodes} & \textbf{\# of Edges} & \textbf{\# of Snapshots} \\
\midrule 
\midrule
Birth-Death {(BD)} & 30K\,\&\,100K   & 12M\,\&\,24M    & 10\,\&\,20 \\
Expand-Contract (EC) & 30K\,\&\,100K   & 13M\,\&\,26M    & 10\,\&\,20 \\
Disappear-Reappear (DR) & 30K\,\&\,100K   & 13M\,\&\,28M    & 10\,\&\,20 \\
Merge-Split (MS) & 30K\,\&\,100K   & 14M\,\&\,29M    & 10\,\&\,20 \\
\midrule
Wikipedia   & 8,400 & 162,000   & 5 \\
Dublin  & 11,000    & 415,900   & 5 \\
arXiv   & 28,100    & 4,600,000 & 5 \\
Stack-Overflow (SO) &2,601,977	&{63,497,050} & 5 \\
Flickr  & 2,302,925 & 33,100,000    & 5    \\
Youtube & {3,200,000} & 12,200,000    & 5    \\
\bottomrule
\end{tabular}}
\label{table1}
\end{table}

\begin{table*}[t]
\footnotesize
\centering
\caption{
Overall performances on dynamic graphs. \textbf{Bold} and \underline{\textit{Underline}} indicate the best 
and second-best performing methods. Symbol $\dag$ indicates that DyG-RoLLM significantly surpassed all 
baselines with a p-value $<0.005$. N/A means that it cannot be executed due to memory and running time constraints.}
\setlength{\tabcolsep}{0.2mm}{
\label{main_table:performance_on_artificial}
\begin{tabular}{@{}lllcccccccccccccccccccccccc@{}}
\toprule
\multicolumn{3}{c}{\multirow{2}{*}{\textbf{\textit{Methods}}}}
& \multicolumn{2}{c}{\textbf{\textcolor{black}{BD-30K}}}
& \multicolumn{2}{c}{\textbf{\textcolor{black}{EC-30K}}}
& \multicolumn{2}{c}{\textbf{\textcolor{black}{BD-100K}}}
& \multicolumn{2}{c}{\textbf{\textcolor{black}{DR-100K}}}
& \multicolumn{2}{c}{\textbf{MS-100K}}
& \multicolumn{2}{c}{\textcolor{black}{\textbf{Wikipedia}}} 
& \multicolumn{2}{c}{\textcolor{black}{\textbf{Dublin}}} 
& \multicolumn{2}{c}{\textcolor{black}{\textbf{arXiv}}} 
& \multicolumn{2}{c}{\textcolor{black}{\textbf{SO}}} 
& \multicolumn{2}{c}{\textbf{Flickr}}
& \multicolumn{2}{c}{\textbf{Youtube}}  
\\ 
\cmidrule(lr){4-5} \cmidrule(lr){6-7} \cmidrule(lr){8-9} \cmidrule(lr){10-11} \cmidrule(lr){12-13} \cmidrule(lr){14-15}  
\cmidrule(lr){16-17} \cmidrule(lr){18-19}  \cmidrule(lr){20-21}  
\cmidrule(lr){22-23} \cmidrule(lr){24-25} 
\multicolumn{3}{c}{} &\textit{NMI\,\,\,} &\textit{\,\,\,NF1\,\,\,}         
&\textit{NMI} &\textit{NF1}  &\textit{NMI}  &\textit{NF1}  &\textit{NMI}  &\textit{NF1} &\textit{NMI}  &\textit{NF1}  &\textit{NMI}  &\textit{NF1}  &\textit{NMI}  &\textit{NF1}  &\textit{NMI}  &\textit{NF1}  &\textit{NMI}  &\textit{NF1}  &\textit{NMI}  &\textit{NF1} &\textit{NMI}  &\textit{NF1} \\
\midrule
\midrule
\multirow{6}{*}{\rotatebox[origin=c]{90}{{\centering {NN-based}}}} 
&&node2vec~\cite{grover2016node2vec}  
& 93.8\,\,\, & \,\,\,85.8\,\,\, & \,\,\,93.2 & 83.6 & 90.3 & 84.2  
& 91.2 & 85.6 & 89.0 & 78.2 & 33.5 & 10.2 
& 50.3 & 23.2 & 44.3 & 26.8 & 40.3 & 23.3 
& 42.4 & 25.5 & 42.8 & 24.5\\
&&RNNGCN~\cite{DBLP:conf/aaai/YaoJ21}   
& 96.6\,\,\, & \,\,\,84.5\,\,\, & \,\,\,95.9 & 85.1 & 90.8  & 84.9 
& 92.1 & 84.3 & 91.2 & 80.9 & 40.3  & {22.5} 
& 52.2 & {31.6} & 45.4 & 26.2 & 40.8  & 23.8 
& {48.5} & 30.2 & 47.6 & 32.3\\
&&SepNE~\cite{li2019sepne}
& 92.5\,\,\, & \,\,\,{89.4}\,\,\, & \,\,\,92.1 & 81.8 & 89.1  & 82.2  
& 89.0 & 81.1 & 89.2 & 78.6 & 31.4 & 9.8  
& 49.2 & 21.0 & 42.8 & 25.8 & 39.2  & 22.1 
& 40.3 & 22.4 & 39.5 & 21.0\\
&&ROLAND~\cite{DBLP:conf/kdd/YouDL22}   
& 95.5\,\,\, & \,\,\,83.2\,\,\, & \,\,\,94.1 & 83.8 & \underline{92.8}  & 85.1 
& {93.3} & {85.8} & {92.8} & {81.6} 
& 42.2  & 22.3 & {53.6} & {31.6} 
& {46.8} & {27.6} & 44.8  & 29.2 
& {47.5} & {31.4} & {48.4}  & {33.2} \\
&&TGC~\cite{DBLP:conf/iclr/001400T00024} 
& 95.3\,\,\, & \,\,\,83.0\,\,\, & \,\,\,93.8 & 83.6 & 92.0  & 84.5  
& 92.8 & 85.4 & 91.5 & 81.2
& 41.3 & 22.3 & 52.8 & 31.3 & 45.8 & 26.8 & 44.6  & 29.0 
& 47.8 & {31.6} & 47.9 & 32.6\\
&&DSCPCD~\cite{10017356} 
&89.2\,\,\, &\,\,\,69.8\,\,\, &\,\,\,88.9 &70.3 & 81.2  & 65.8  
&79.6 &63.5 &81.0 &64.4 
&28.2 &7.3 &32.4 &11.9 &31.8 &10.9  & 34.4  & 18.6 
&34.3 & 18.2 & 32.6 &15.9 \\
\midrule
\multirow{6}{*}{\rotatebox[origin=c]{90}{{\centering {MF-based}}}} 
&&PisCES~\cite{LiuF1801} 
& 91.2\,\,\, & \,\,\,41.6\,\,\, & \,\,\,92.6 & 49.0 &N/A  &N/A  
&N/A &N/A &N/A &N/A  
& 32.1 & 9.9  & 46.3 & 16.2 & 38.2 & 14.5 &N/A  &N/A 
& N/A & N/A  & N/A & N/A     \\
&&DYNMOGA~\cite{Folina2013} \,\,
& \underline{98.1}\,\,\, & \,\,\,78.1\,\,\, & \,\,\,{98.2} & 65.3 &N/A  &N/A   
&N/A &N/A &N/A &N/A
& 36.2 & 9.9  & 49.8 & 20.1  & 39.1 & 24.5 &N/A  &N/A 
& N/A   & N/A & N/A & N/A \\
&&NE2NMF~\cite{DBLP:journals/kbs/LiZDGM21} 
& 97.1\,\,\, & \,\,\,76.1\,\,\, & \,\,\,97.5 & 63.1 &N/A  &N/A  
&N/A &N/A &N/A &N/A
& 34.1 & 8.2  & 47.9 & 18.9  & 38.2 & 22.9 &N/A  &N/A 
& N/A   & N/A & N/A & N/A\\
&&RTSC~\cite{DBLP:conf/cikm/YouHKFO21} 
& 92.8\,\,\, & \,\,\,55.3\,\,\, & \,\,\,92.1 & 53.2 &N/A  &N/A  
&N/A &N/A &N/A &N/A
& 30.6 & 11.3 & 46.6 & 19.3 & 38.2 & 20.2  &N/A  &N/A 
& N/A   & N/A & N/A & N/A \\
&&RDMA~\cite{DBLP:journals/www/RanjkeshMH24} 
& 95.3\,\,\, & \,\,\,69.8\,\,\, & \,\,\,94.8 & {85.5} &N/A  &N/A 
&N/A &N/A &N/A &N/A & 33.8 &10.2 
& 47.2 & 18.6 &41.6 & 25.2 &N/A  &N/A  
&N/A &N/A &N/A &N/A\\
&&DyG-MF~\cite{li2025revisiting}
& $\textbf{99.9}$\,\,\, 
& \,\,\,\underline{90.2}\,\,\,  
& \,\,\,\underline{99.2}  
& \underline{90.9}  
& {92.3}  
& \underline{86.8}  
& \underline{94.3} 
& \underline{86.5}  
& \underline{94.4}  
& \underline{83.2}  
& \underline{50.4} 
& \underline{25.8} 
& \underline{56.1}
& \underline{33.7} 
& \underline{51.8} 
& \underline{30.2} 
& \underline{45.5}  
& \underline{29.6} 
& \underline{52.3} 
& \underline{33.6} 
& \underline{51.8} 
& \underline{34.5} 
\\ 
\midrule
\multirow{3}{*}{\rotatebox[origin=c]{90}{{\centering {LLMs}}}} 
&&ComGPT~\cite{ni2024comgpt} 
& 90.4\,\,\, & \,\,\,70.2\,\,\, & \,\,\,90.2 & 71.3 
& 82.5 & 65.9 & 81.2 & 65.4 & 83.4 & 65.5 
& 30.5 & 8.9 & 45.9 & 19.4 & 35.5 & 16.7 
& 38.2  & 20.1  & 38.2 & 20.6 & 36.2 & 20.6
\\
&&CommLLM~\cite{gujral2025llms} 
& 88.3\,\,\, & \,\,\,68.7\,\,\, & \,\,\,86.2 & 69.1 
& 79.3 & 63.2 & 75.8 & 60.1 & 79.2 & 62.3 
& 26.4 & 7.2  & 38.8 & 14.4 & 30.3 & 10.5 
& 32.2 & 17.3 & 33.3 & 18.2 & 30.8 & 15.1
\\
&\,\,\quad&DyG-RoLLM
& $\textbf{99.9}$\,\,\,  
& \,\,\,$\textbf{90.4}$\,\,\,  
& \,\,\,$\textbf{99.4}$  
& $\textbf{\,\,\,91.7}^{\dag}$  
& $\textbf{\,\,\,94.3}^{\dag}$ 
& $\textbf{\,\,\,90.9}^{\dag}$ 
& $\textbf{\,\,\,96.2}^{\dag}$ 
& $\textbf{\,\,\,88.6}^{\dag}$  
& $\textbf{\,\,\,95.4}^{\dag}$  
& $\textbf{\,\,\,84.8}^{\dag}$  
& $\textbf{\,\,\,51.6}^{\dag}$ 
& $\textbf{\,\,\,26.7}^{\dag}$ 
& $\textbf{\,\,\,57.6}^{\dag}$
& $\textbf{\,\,\,35.2}^{\dag}$ 
& $\textbf{\,\,\,53.3}^{\dag}$ 
& $\textbf{\,\,\,32.1}^{\dag}$ 
& $\textbf{\,\,\,47.6}^{\dag}$ 
& $\textbf{\,\,\,30.9}^{\dag}$ 
& $\textbf{\,\,\,54.2}^{\dag}$ 
& $\textbf{\,\,\,34.8}^{\dag}$ 
& $\textbf{\,\,\,52.7}^{\dag}$ 
& $\textbf{\,\,\,35.3}^{\dag}$ 
\\
\bottomrule
\end{tabular}}
\end{table*}

\section{Experiment}
\label{experiments}

\noindent \textbf{Datasets.} 
As shown in Table~{\ref{table1}}, following previous works~\cite{DBLP:journals/kbs/LiZDGM21,DBLP:conf/cikm/YouHKFO21}, we evaluate baselines on four synthetic dynamic graphs and six real-world dynamic graphs with varying number of nodes and edges. 
Synthetic Green datasets~\cite{Folina2013} generate dynamic graphs considering four evolution events including  \textit{Birth-Death (BD)}: existing clusters are removed or generated by randomly selecting nodes from other clusters;  
\textit{Expand-Contract (EC)}: clusters are expanded or contracted;
\textit{Disappear-Reappear (DR)}: clusters are randomly hidden and reappeared; 
\textit{Merge-Split (MS)}: clusters are split or merged. We also evaluate on various real-world domains, including \textit{Academic Graphs}: arXiv~\cite{DBLP:conf/kdd/LeskovecKF05}; \textit{Social Graphs}: Dublin~\cite{ISELLA2011166} and Flickr~\cite{mislove-2008-flickr}
and \textit{Website Interaction Graphs}: Wikipedia~\cite{DBLP:conf/icwsm/LeskovecHK10}, Stack-Overflow (SO)~\cite{paranjape2017motifs}, and Youtube~\cite{mislove-2008-flickr}.

\noindent \textbf{Baselines and Backbones.} 
We compare DyG-RoLLM with 14 best-performing baselines, \textit{i.e.,} Neural Network (NN)-based methods: 
node2vec~\cite{grover2016node2vec},
RNNGCN~\cite{DBLP:conf/aaai/YaoJ21}, 
SepNE~\cite{li2019sepne}, 
ROLAND~\cite{DBLP:conf/kdd/YouDL22}, 
TGC~\cite{DBLP:conf/iclr/001400T00024}, 
and DSCPCD~\cite{10017356}; 
and Matrix Factorization (MF)-based methods: PisCES~\cite{LiuF1801}, 
DYNMOGA~\cite{Folina2013}, NE2NMF~\cite{DBLP:journals/kbs/LiZDGM21}, 
RTSC~\cite{DBLP:conf/cikm/YouHKFO21}, RDMA~\cite{DBLP:journals/www/RanjkeshMH24}, 
and DyG-MF~\cite{li2025revisiting};
LLMs-based methods: ComGPT~\cite{ni2024comgpt} and CommLLM~\cite{gujral2025llms}.

\noindent \textbf{Implementation Details.} Following previous works~\cite{LiuF1801,Ma17b}, we use normalized mutual information (NMI)~\cite{Danon_2005} and normalized F1-score (NF1)~\cite{rossetti2017graph} to measure clustering accuracy.
We reproduced baselines using their optimal parameters and reported average performance over five repeated runs with different random seeds. 
We adopt GPT-4o as backbone LLM.
We conducted multiple t-tests with Benjamini-Hochberg~\cite{benjam} correction to assess the statistical significance. 
We took DR-30K as validation set for hyperparameter tuning. 
With grid search, our method achieves the best results with embedding sizes $d=512$, $d_r=128$, $d_c=384$; 
sharpness parameters $\tau_c = \tau_r = 0.5$; and regularization weights $\lambda_1 = 0.1$ and $\lambda_2 = 0.2$.  

\begin{figure}[t]
\centering
\includegraphics[scale=0.26]{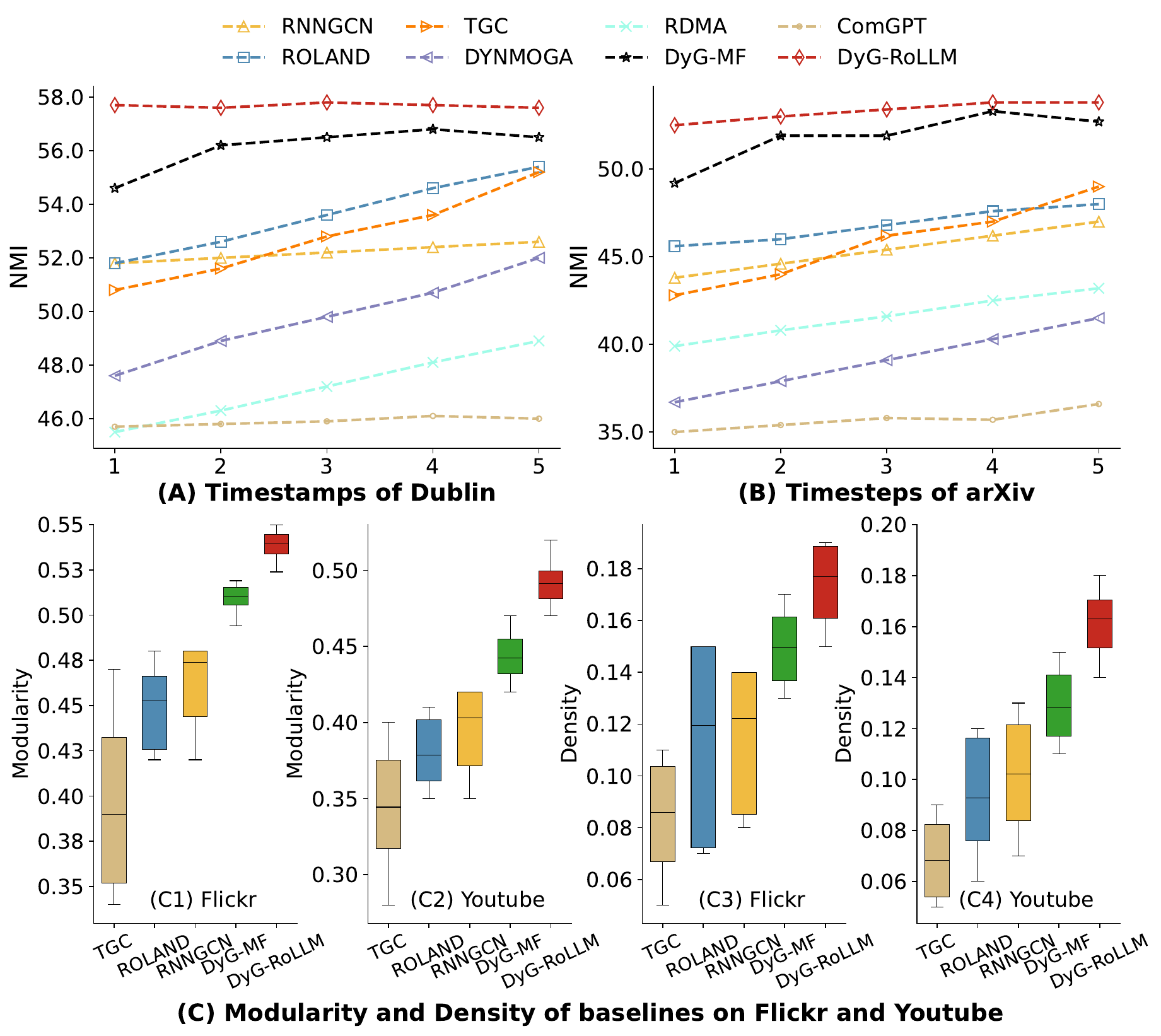}
\caption{Performance comparison of (A-B) varying timestamps and (C) modularity and density of selected baselines. }
\label{main_fig:SMFT2_accuracy}
\end{figure}

\subsection{Performance Evaluation} 
As shown in Table~\ref{main_table:performance_on_artificial}, we report the performance comparison across synthetic and real-world dynamic graphs. DyG-RoLLM consistently achieves the highest NMI and NF1 scores, demonstrating the effectiveness of clustering-aware graph encoding, interpretable node-role learning, and hierarchical LLM reasoning. 
\textbf{First}, compared to NN-based methods RNNGCN and ROLAND, DyG-RoLLM improves NMI scores by 12.9\% and 9.8\% on Flickr and Youtube, respectively. 
This improvement stems from our joint optimization of node embeddings and modularity-based clustering, where orthogonal decomposition ensures that community-specific features are learned independently from role patterns, providing optimally separated representations for clustering tasks.
\textbf{Second}, compared to the best-performing MF-based method DyG-MF, DyG-RoLLM achieves 3\% and 4\% improvements in NMI and NF1 on real-world graphs. 
While MF-based methods produce opaque latent factors that conflate structural proximity with functional similarity, our method leverages interpretable role prototypes to guide LLM reasoning, ensuring both structural coherence through modularity optimization and semantic validity through logical consistency checks.
\textbf{Third}, compared to LLM-based methods, DyG-RoLLM improves NMI by 48.64\% over ComGPT and 66.92\% over CommLLM on average. LLM methods rely solely on language models without grounding in graph structure, producing semantically plausible but topologically incorrect clusters. {Our bidirectional learning ensures LLM reasoning is constrained by structural evidence while GNN optimization benefits from semantic guidance, creating interaction that neither pure GNN nor LLM methods can achieve alone.}

As shown in Figure~\ref{main_fig:SMFT2_accuracy} (A-B), we evaluate temporal performance across all timestamps. DyG-RoLLM demonstrates remarkable stability, maintaining superior performance throughout the entire evolution period with minimal variance, crucial for real-world deployment. 
Furthermore, we employ two label-free metrics, Modularity~\cite{Newman06} and Conductance~\cite{chen2013measuring}, to evaluate community quality without ground-truth dependency. This is essential because NMI and NF1 require potentially noisy or incomplete labels. Specifically, modularity measures intra-community edge density relative to random expectation, while conductance quantifies the cut quality between communities. Figure~\ref{main_fig:SMFT2_accuracy} (C) shows that DyG-RoLLM consistently outperforms the four best-performing baselines on both metrics across real-world graphs, validating its ability to identify structurally coherent communities independent of label quality.

\begin{figure}[h]
\centering
\includegraphics[scale=0.24]{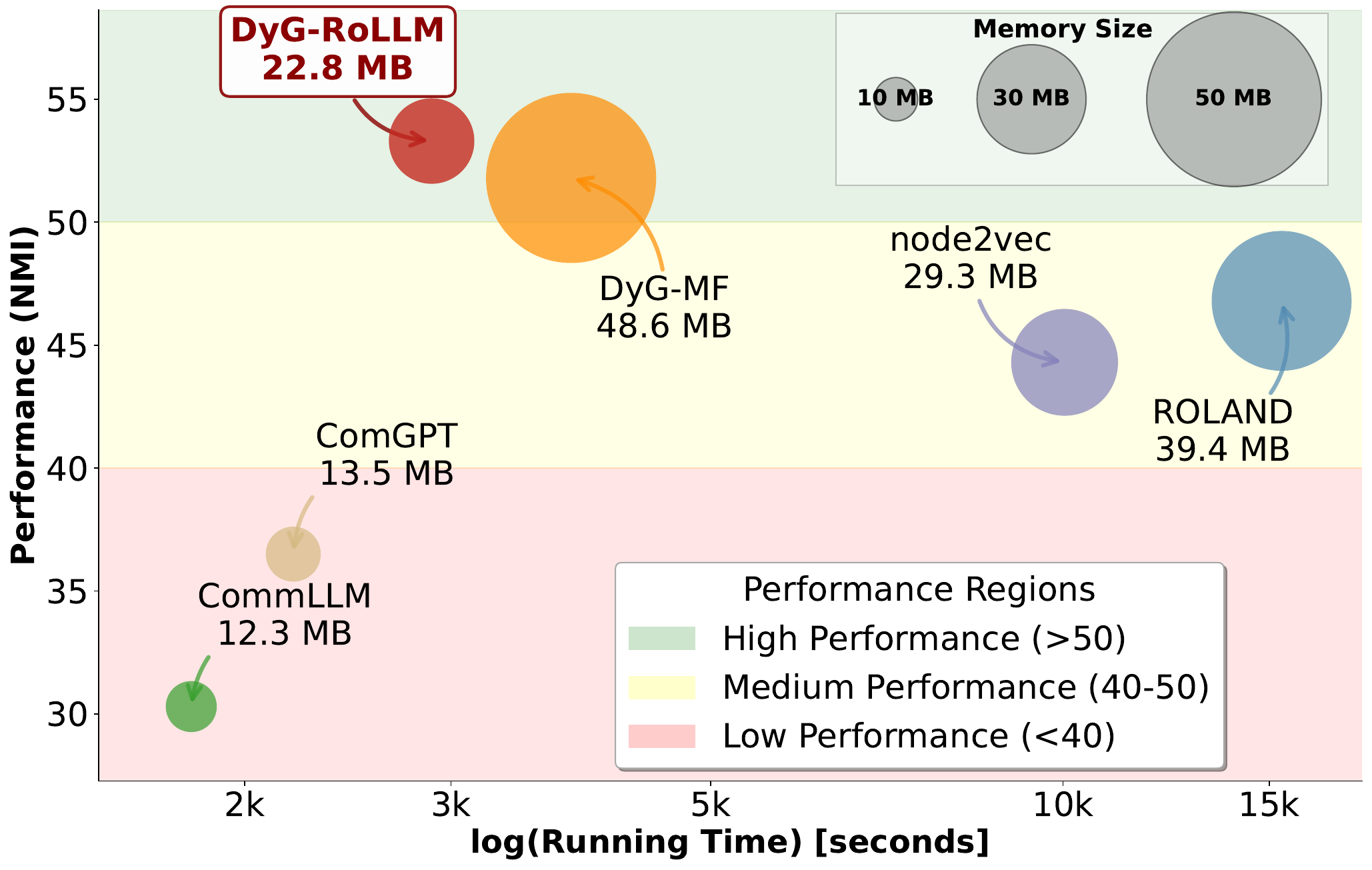}
\caption{Efficiency comparison on large dynamic graphs.}
\label{fig:efficiency}
\end{figure}

\subsection{Efficiency Evaluation}

Figure~\ref{fig:efficiency} illustrates the efficiency-effectiveness trade-off across different baselines. DyG-RoLLM demonstrates a superior balance between computational efficiency and clustering performance. Despite the minimal parameter footprint of LLM-based approaches (ComGPT and CommLLM), their clustering accuracy remains inadequate for practical applications. While DyG-MF achieves competitive performance, its computational overhead and memory footprint limit scalability. Traditional neural approaches exhibit prohibitive resource consumption that precludes large-scale deployment. DyG-RoLLM uniquely combines high clustering accuracy with reasonable computational requirements, making it suitable for real-world applications. These results, showed on the medium-scale arXiv dataset, generalize consistently across all evaluated datasets.

\subsection{Interpretability Evaluation}

To validate our interpretable design, we conduct comprehensive evaluations through quantitative metrics and human assessment.

\subsubsection{Quantitative Metrics}

We evaluate interpretability through two metrics. First, we define {Role Consistency Score (RCS)} that measures whether node roles remain stable over time as:
\begin{footnotesize}
\begin{equation}
\text{RCS} = 1 - \frac{1}{|V^t{\cap}V^{t+1}|(T-1)} \sum_{i \in V^t{\cap}V^{t+1}} \sum_{t=2}^{T} \|\boldsymbol{\pi}_i^t - \boldsymbol{\pi}_i^{t-1}\|_1,
\end{equation}
\end{footnotesize}where RCS=1 denotes stable and meaningful roles. 

Second, we define {Explanation Fidelity Score (EFS)} to verify whether explanations match actual graph structure as:
\begin{footnotesize}
\begin{equation}
\text{EFS} = \frac{1}{N} \sum_{i=1}^{N} \frac{|\text{verified claims}_i|}{|\text{total claims}_i|}
\end{equation}
\end{footnotesize}where $N$ denotes the number of randomly selected structural claims and we verify them against the graph. EFS=0.9 means 90\% of explanatory claims are correct.

\begin{table}[t]
\centering
\footnotesize
\caption{Interpretability on real-world dynamic graphs.}
\label{tab:interpretability}
\begin{tabular}{lcccccc}
\toprule
\multirow{2}{*}{\textbf{Method}} & \multicolumn{2}{c}{\textbf{Wikipedia}} & \multicolumn{2}{c}{\textbf{Dublin}} & \multicolumn{2}{c}{\textbf{SO}} \\
\cmidrule(lr){2-3} \cmidrule(lr){4-5} \cmidrule(lr){6-7}
 & RCS↑ & EFS↑ & RCS↑ & EFS↑ & RCS↑ & EFS↑ \\
\midrule
ComGPT (+CoT) & 0.69 & 0.58 & 0.64 & 0.55 & 0.71 & 0.60 \\
CommLLM (+CoT)  & 0.62 & 0.51 & 0.58 & 0.47 & 0.65 & 0.53 \\
DyG-RoLLM (w/o Role) & 0.75 & 0.66 & 0.68 & 0.68 & 0.83 & 0.82 \\
DyG-RoLLM (w/o CoT) & 0.85 & 0.64 & 0.82 & 0.66 & 0.87 & 0.68 \\
DyG-RoLLM  & \textbf{0.88} & \textbf{0.89} & \textbf{0.88} & \textbf{0.90} & \textbf{0.92} & \textbf{0.91} \\
\bottomrule
\end{tabular}
\end{table}

Since baselines lack interpretability, we adapt them for fair comparison. 
For \textbf{ComGPT (+CoT)} and \textbf{CommLLM (+CoT)}, we enhance them with CoT prompting to generate step-by-step reasoning about role and cluster assignments.  
For \textbf{DyG-RoLLM (w/o Role)}, we remove role description but retain LLM reasoning. 
\textbf{DyG-RoLLM (w/o CoT)} keeps roles but uses direct assignment without CoT reasoning. 
As shown in Table~\ref{tab:interpretability}, DyG-RoLLM achieves superior interpretability by two key advantages. First, our method significantly outperforms LLM-only baselines on RCS because explicit role prototypes provide stable semantic anchors that prevent random explanation fluctuations. Second, DyG-RoLLM attains EFS>0.89, while enhanced baselines remain below 0.7. 
This gap reveals that pure LLM methods generate plausible-sounding but factually incorrect explanations. The performance drop in our ablations confirms that both role prototypes and structured reasoning are essential.

\begin{table}[h]
\centering
\footnotesize
\caption{Human evaluation with average scores on 1-5 scale.}
\label{tab:human_eval}
\setlength{\tabcolsep}{1.4mm}{
\begin{tabular}{lccccc}
\toprule
\textbf{Method} & \textbf{Clarity} & \textbf{Soundness} & \textbf{Complete} & \textbf{Trust} & \textbf{Action} \\
\midrule
ComGPT (+CoT) & 3.2 & 2.8 & 2.6 & 2.9 & 2.5 \\
CommLLM (+CoT) & 2.9 & 2.5 & 2.4 & 2.6 & 2.3 \\
DyG-RoLLM (w/o Role) & 3.5 & 3.3 & 3.1 & 3.4 & 3.2 \\
DyG-RoLLM (w/o CoT) & 3.8 & 3.6 & 3.4 & 3.7 & 3.5 \\
DyG-RoLLM (full) & \textbf{4.4} & \textbf{4.5} & \textbf{4.3} & \textbf{4.6} & \textbf{4.2} \\
\bottomrule
\end{tabular}}
\end{table}

\subsubsection{Human Evaluation} Three domain experts (doctoral students in graph mining) independently evaluated 50 randomly selected explanations on Wikipedia using a 5-point Likert scale (1=poor, 5=excellent) across five criteria: \textbf{Clarity} (comprehensibility), \textbf{Soundness} (logical validity), \textbf{Completeness} (coverage of relevant factors), \textbf{Trustworthiness} (decision confidence), and \textbf{Actionability} (intervention insights). Table~\ref{tab:human_eval} shows that DyG-RoLLM achieves higher scores on all criteria, with Clarity=4.4 and Soundness=4.5, approaching "excellent" ratings. 
The high trustworthiness indicates that experts can trust these explanations for decision making. Enhanced baseline score below 3.2, confirming that LLM reasoning without structural grounding produces less reliable explanations.

\begin{figure}[h]
\centering
\includegraphics[scale=0.4]{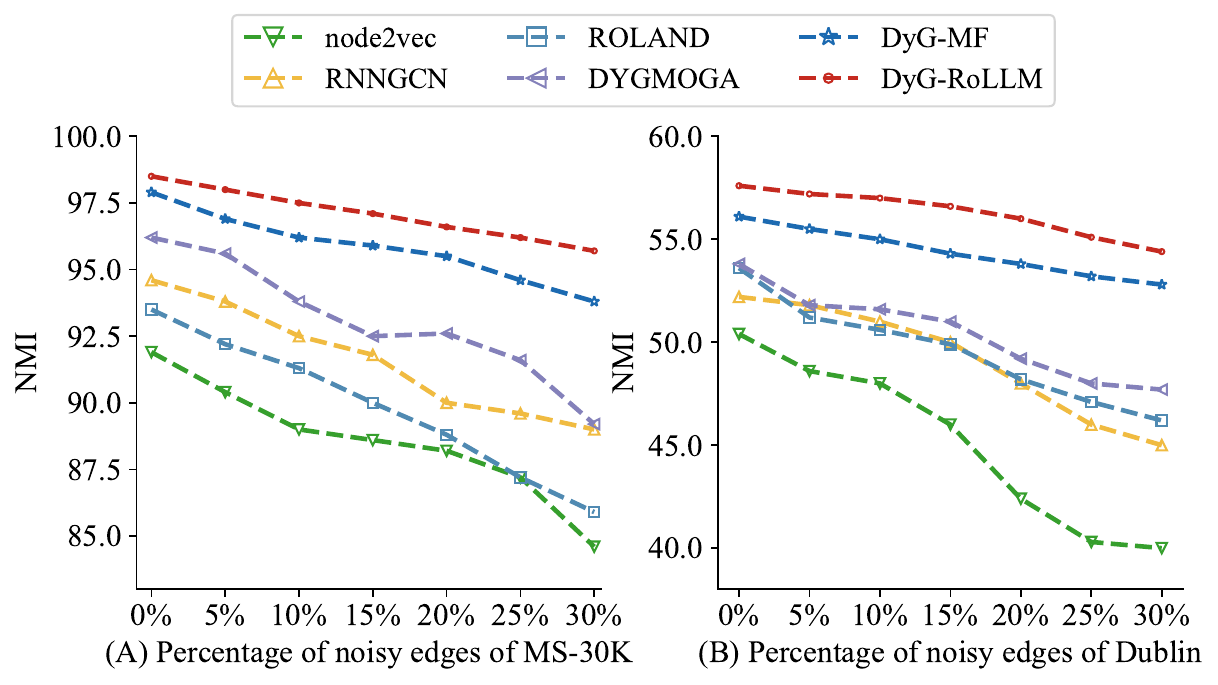}
\caption{Noise Attacks. NMI w.r.t. percentage of noisy edges.}
\label{fig:noisy}
\end{figure}

\subsection{Robustness Evaluation}

Real-world dynamic graphs often contain noise and exhibit irregular evolution patterns. To evaluate robustness, we contaminate graphs by injecting 5\%-30\% random edges per snapshot following~\cite{DBLP:conf/cikm/TanYL22}. 
As shown in Figure~\ref{fig:noisy}, DyG-RoLLM demonstrates superior noise resilience in both datasets. On MS-30K/Dublin, DyG-RoLLM experiences only 2.8\%/5.6\% NMI degradation in 30\% noise compared to 4.2\%/5.9\% for DyG-MF. This robustness stems from three key designs: (1) orthogonal decomposition isolates role learning from noisy community boundaries, preventing error propagation; (2) role prototypes provide stable semantic anchors that filter spurious variations; (3) LLM reasoning validates assignments through logical consistency, rejecting noise-induced implausible structures. The consistent superiority across varying noise levels confirms DyG-RoLLM's practical applicability in noisy real-world scenarios.

\begin{table}[h]
\centering
\footnotesize
\caption{Ablation and strategy replacement of DyG-RoLLM.}
\label{tab:ablation}
\setlength{\tabcolsep}{2.5mm}{
\begin{tabular}{clcccc}
\toprule
\multicolumn{2}{c}{\multirow{2}{*}{\textbf{\textit{Methods}}}}  
& \multicolumn{2}{c}{\textbf{Wikipedia}} & \multicolumn{2}{c}{\textbf{arXiv}}  \\
\cmidrule(lr){3-4} \cmidrule(lr){5-6} &
 & NMI↑ & NF1↑ & NMI↑ & NF1↑  \\
\midrule
\multirow{7}{*}{\rotatebox[origin=c]{90}{{\centering {Remove}}}} 
&w/o Orthogonal Subspace & 42.2 & 22.4 & 43.5 & 15.8 \\
&w/o Clustering loss Eq.(\ref{clustering-loss}) & 50.3 & 25.6 & 51.6 & 30.0 \\
&w/o Temporal loss Eq.(\ref{smoothing-loss}) & 51.0 & 26.3 & 52.5  & 31.7  \\
&w/o Contrastive loss Eq.(\ref{contrastive-loss}) & 50.6 & 25.7 & 51.7 & 30.1  \\
&w/o Consistent loss Eq.(\ref{consistant-loss}) & 50.0 & 25.1 & 51.2 & 29.7 \\
&w/o Role-Guided Description & 48.3 & 24.1 & 46.2 & 16.4  \\
&w/o LLM Reasoning & 47.2 & 23.8 & 45.1 & 16.0  \\
\midrule
\multirow{4}{*}{\rotatebox[origin=c]{90}{{\centering {Replace}}}} 
&Random Prompt Init & 51.0 & 26.4 & 52.9 & 32.0  \\
&GPT-3.5-turbo & 50.3 & 25.9 & 52.0 & 31.4  \\
&Claude-3.5 & 50.8 & 26.0 & 52.6 & 32.0  \\
&Llama2-7B & 50.1 & 25.8 & 51.9 & 31.2  \\
\midrule
&DyG-RoLLM (full)  & \textbf{51.6} & \textbf{26.7} & \textbf{53.3} & \textbf{32.1} \\
\bottomrule
\end{tabular}}
\end{table}

\subsection{Ablation Study}
\label{sec:ablation_study}

We conduct comprehensive ablation studies to evaluate the necessity of each component in DyG-RoLLM. we adopt seven variants: \textit{w/o Orthogonal Subspace} denotes using a unified embedding for both clustering and prototype learning instead of separate subspaces; \textit{w/o Role-guided Description} removes the three-level semantic descriptions for each node; and \textit{w/o LLM Reasoning} bypasses CoT reasoning and directly predicts community assignments. The remaining variants ablate individual loss functions to assess their contributions. As shown in Table~{\ref{tab:ablation}}, removing any of these components negatively impacts performance on graph clustering, demonstrating their effectiveness and necessity.
Furthermore, we conduct four strategy replacement experiments using random prototype initialization and different LLM backbones. The results demonstrate remarkable robustness: performance degradation remains minimal across all alternatives, indicating that DyG-RoLLM is not overly dependent on specific implementation choices and can be adapted to various deployment scenarios with different computational constraints.

\begin{figure}[h]
\centering
\includegraphics[scale=0.22]{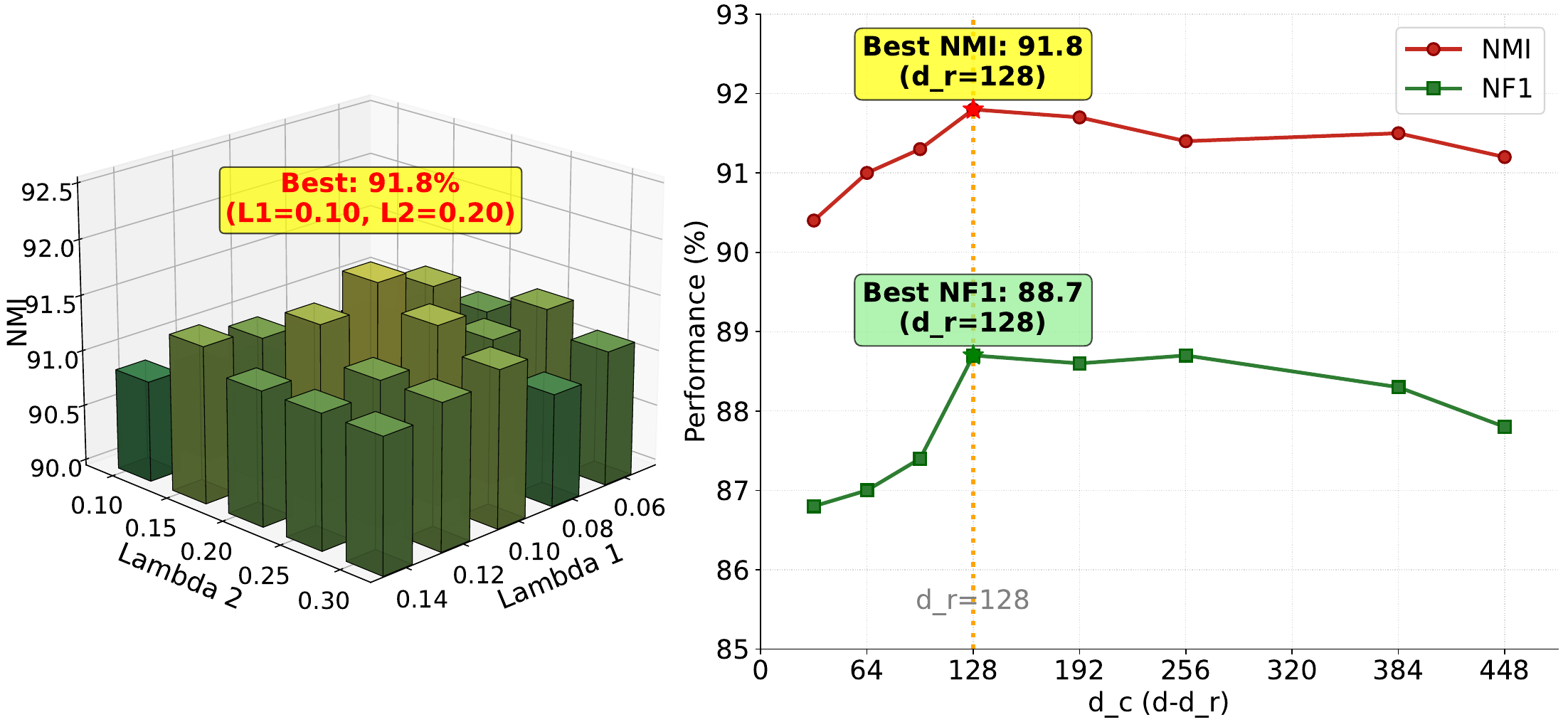}
\caption{Hyperparameter tuning results.}
\label{Hyper_parameter_tuning}
\end{figure}

\subsection{Hyperparameter Sensitivity Analysis}

We conducted sensitivity analysis for critical hyperparameters $\lambda_1$ and $\lambda_2$, as well as the dimensions of cluster space $d_c$ and role space $d_r$ on the DR-30K dataset. As shown in Figure~\ref{Hyper_parameter_tuning}, DyG-RoLLM achieves optimal performance at $\lambda_1=0.1$ and $\lambda_2=0.2$. Increasing either parameter, which amplifies the weights of temporal and diversity components respectively, degrades model performance. These values were fixed for all experiments. Regarding dimensional allocation, the model performs best with $d_r=128$ and $d_c=384$, suggesting that cluster representation requires more features than role description. Notably, performance remains stable across various dimensional combinations, demonstrating the robustness of our approach and its practical applicability across different settings.

\section{Conclusion}
\label{conclusion}

In this paper, we presented DyG-RoLLM, a framework for interpretable dynamic graph clustering. By introducing orthogonal subspace decomposition to separate role identification from clustering, we solve the conflict between these tasks. Through learnable prototypes as semantic anchors, we transform graph structures into interpretable narratives enabling LLM reasoning about communities. 
Future work will explore continuous-time dynamic graphs, incorporating multimodal node attributes.

\section*{Acknowledgment}
This work was supported by the UTokyo Global Activity Support Program for Young Researchers. We sincerely thank our peers and reviewers for their valuable feedback to enhance our work.

\bibliographystyle{ACM-Reference-Format}
\bibliography{sample-base}

\appendix

\section{Proof of Proposition~\ref{thm:gradient_isolation}}
\label{app:proof_gradient_isolation}

\begin{proof}
We prove gradient isolation by showing that losses in orthogonal subspaces have zero cross-gradients. 
Let \begin{footnotesize}$\mathcal{F}(\mathbf{x}) = [\mathbf{W}_1 \mathbf{x}\,\|\, \mathbf{W}_2 \mathbf{x}]$\end{footnotesize} with $\mathbf{W}_1 \mathbf{W}_2^{\top} = \mathbf{0}$ and $\mathcal{L} = \mathcal{L}_1(\mathbf{W}_1\mathbf{x}) + \mathcal{L}_2(\mathbf{W}_2\mathbf{x})$. Since $\mathcal{L}_1$ depends only on $\mathbf{z}_1 = \mathbf{W}_1\mathbf{x}$, we have 
\begin{equation}
\nabla_{\mathbf{W}_2} \mathcal{L}_1 = \frac{\partial \mathcal{L}_1}{\partial \mathbf{z}_1} \cdot \frac{\partial (\mathbf{W}_1\mathbf{x})}{\partial \mathbf{W}_2} = \mathbf{0}
\end{equation}
Similarly, $\nabla_{\mathbf{W}_1} \mathcal{L}_2 = \mathbf{0}$.
Moreover, for any parameter $\theta$ affecting $\mathbf{x}$:
\begin{equation}
\nabla_{\theta} \mathcal{L}_1 = \frac{\partial \mathcal{L}_1}{\partial \mathbf{z}_1} \mathbf{W}_1 \frac{\partial \mathbf{x}}{\partial \theta}, \quad
\nabla_{\theta} \mathcal{L}_2 = \frac{\partial \mathcal{L}_2}{\partial \mathbf{z}_2} \mathbf{W}_2 \frac{\partial \mathbf{x}}{\partial \theta}
\end{equation}

\noindent The orthogonality $\mathbf{W}_1\mathbf{W}_2^{\top} = \mathbf{0}$ ensures gradient orthogonality:
\begin{equation}
\langle \nabla_\theta \mathcal{L}_1, \nabla_\theta \mathcal{L}_2 \rangle \propto \text{tr}(\mathbf{W}_1\mathbf{W}_2^{\top}) = 0
\end{equation}

\noindent Therefore, $\nabla_{\mathbf{W}_2} \mathcal{L}_1 = \mathbf{0}$ and $\nabla_{\mathbf{W}_1} \mathcal{L}_2 = \mathbf{0}$, enabling independent optimization without interference.
\end{proof}

\section{Initialization of Node-Role Prototypes}
\label{AppendB-initialization}

As shown in Table~\ref{tab:prototype_prior_matrix_revised}, we initialize $\mathbf{f}^{\text{prior}}$ using binary assignments (1.0 for high intensity, 0.0 for low intensity) to ensure semantic distinctiveness. This sparse initialization serves two purposes:

\begin{itemize}
\item \textbf{Semantic separation:} Binary values force maximum separation in embedding space, preventing feature redundancy and accelerating convergence of the prototype alignment loss $\mathcal{L}_{\text{proto}}$.

\item \textbf{Role disambiguation:} Contradictory feature assignments prevent semantic confusion-Newcomers have high volatility but zero centrality, distinguishing them from stable high-degree nodes.
\end{itemize}

This initialization strategy provides semantically grounded starting points essential for stable multi-level prototype learning. The specific features—centrality metrics (degree, betweenness, closeness), local structure indicators (clustering coefficient, intra community density), and temporal dynamics (volatility, stability)—are standard graph metrics widely used in network analysis.

\begin{table*}[h]
\centering
\caption{The Node Role Prototype Prior Matrix ($\mathbf{f}^{\text{prior}}$). The matrix encodes the ideal structural and temporal characteristics of the five predefined roles. A value of $\mathbf{1}$ indicates a required high intensity for the role, while $\mathbf{0}$ indicates a low intensity.}
\label{tab:prototype_prior_matrix_revised}
\resizebox{1\textwidth}{!}{%
\begin{tabular}{l|c|c|c|c|c|c|c}
\toprule
\textbf{Prototype ($\mathbf{p}_i^n$)} & \multicolumn{3}{c|}{\textbf{Centrality ($\mathbf{f}^{\mathbf{C}}$)}} & \multicolumn{2}{c|}{\textbf{Local Structure ($\mathbf{f}^{\mathbf{L}}$)}} & \multicolumn{2}{c}{\textbf{Temporal ($\mathbf{f}^{\mathbf{T}}$)}} \\
\cmidrule(lr){2-4} \cmidrule(lr){5-6} \cmidrule(lr){7-8}
& \textbf{Degree} & \textbf{Betweenness} & \textbf{Closeness} & \textbf{Clustering Coeff.} & \textbf{Intra-Density} & \textbf{Volatility} & \textbf{Stability} \\
\midrule
\textit{Community Leader} ($\mathbf{p}_1^n$) & $\mathbf{1}$ & $\mathbf{0}$ & $\mathbf{1}$ & $\mathbf{1}$ & $\mathbf{1}$ & $\mathbf{0}$ & $\mathbf{1}$ \\
\textit{Active Member} ($\mathbf{p}_2^n$) & $\mathbf{1}$ & $\mathbf{0}$ & $\mathbf{0}$ & $\mathbf{1}$ & $\mathbf{1}$ & $\mathbf{0}$ & $\mathbf{1}$ \\
\textit{Peripheral Member} ($\mathbf{p}_3^n$)& $\mathbf{0}$ & $\mathbf{0}$ & $\mathbf{0}$ & $\mathbf{0}$ & $\mathbf{0}$ & $\mathbf{1}$ & $\mathbf{0}$ \\ 
\textit{Bridge Node} ($\mathbf{p}_4^n$) & $\mathbf{1}$ & $\mathbf{1}$ & $\mathbf{0}$ & $\mathbf{0}$ & $\mathbf{0}$ & $\mathbf{0}$ & $\mathbf{1}$ \\
\textit{Newcomer/Evolver} ($\mathbf{p}_5^n$) & $\mathbf{0}$ & $\mathbf{0}$ & $\mathbf{0}$ & $\mathbf{0}$ & $\mathbf{0}$ & $\mathbf{1}$ & $\mathbf{0}$ \\
\bottomrule
\end{tabular}
}
\end{table*}

\section{Community Evolution Patterns}
\label{appendix:events}

We identify six fundamental evolution patterns based on role distribution changes $\Delta\boldsymbol{\gamma}_{k}^t = \boldsymbol{\gamma}_{k}^{t+1} - \boldsymbol{\gamma}_{k}^{t}$:

\begin{itemize}
\item \textbf{Birth}: New community emerges at $t+1$ with $|C_k^{t+1}| > \tau_{\text{min}}$ and $|C_k^t| = 0$, typically initiated by Newcomers clustering together;

\item \textbf{Death}: Community disappears with $|C_k^t| > 0$ and $|C_k^{t+1}| = 0$;

\item \textbf{Growth}: $\Delta\gamma_{k,5} > 0.3$ (Newcomer influx) with $|\Delta\gamma_{k,1}| < 0.1$ (stable leadership), indicating healthy expansion phase;

\item \textbf{Contraction}: $\Delta\gamma_{k,3} < -0.3$ (Wanderer exodus) with $\Delta\gamma_{k,1} > 0.1$ (leadership concentration), suggesting core group consolidation;

\item \textbf{Split}: $\Delta\gamma_{k,4} < -0.5$ (Connector loss) and emergence of multiple components from $C_k^t$, reflecting internal fragmentation;

\item \textbf{Merge}: $\Delta\gamma_{k,4} > 0.3$ (Connector increase) with community count reduction from $t$ to $t+1$, showing successful integration;

\end{itemize}

These thresholds are empirically determined: $\tau_{\text{min}} = 5$ (minimum community size), role change thresholds of 0.3 capture significant shifts while filtering noise, and the Connector threshold of 0.5 for splits ensures only major bridge losses trigger split detection. These patterns enable LLMs to reason about community dynamics through interpretable role-based changes.

\section{Semantic Description Generation}
\label{append-prompt1}

We design template function $\mathcal{T}(\cdot)$ to convert numerical graph metrics into hierarchical natural language descriptions. Box D presents the template specification that transforms role distributions, community compositions, and evolution patterns into interpretable narratives for LLM reasoning. The template ensures consistent generation of interpretable descriptions that preserve mathematical relationships while enabling natural language understanding for downstream LLM reasoning. 

\section{Time Complexity Analysis}
\label{append-complexity}

\noindent \textbf{Time Complexity.}
The overall time complexity per snapshot is $\mathcal{O}(L|E^t|d + |V^t|Kd_c + |V^t|d_r)$, where $L$ is the number of GAT layers, $d$, $d_c$, $d_r$ are the dimensions of embedding, cluster, and role subspaces respectively, and $K$ is the number of communities. Specifically:
\begin{itemize}
\item GAT encoding: $\mathcal{O}(L|E^t|d)$ for $L$ layers with hidden dimension $d$.
\item Modularity optimization: $\mathcal{O}(|E^t| + |V^t|d_c)$ via our decomposition.
\item Prototype learning: $\mathcal{O}(|V^t|d_r)$ for role assignments.
\item LLM reasoning: $\mathcal{O}(|V^t|)$ API calls (can be batched).
\end{itemize}

Critically, our modularity reformulation reduces complexity from $\mathcal{O}(|V^t|^2)$ to $\mathcal{O}(|E^t|)$, enabling scaling to graphs with millions of nodes. For sparse graphs where $|E^t| = \mathcal{O}(|V^t|)$, the total complexity becomes linear in the number of nodes.

\noindent \textbf{Space Complexity.}
The space requirement is $\mathcal{O}(|V^t|(d + d_c + d_r) + Kd_c)$, consisting of:
\begin{itemize}
\item Node embeddings: $\mathcal{O}(|V^t|d)$ for temporal states.
\item Orthogonal projections: $\mathcal{O}(|V^t|(d_c + d_r))$. 
\item Cluster centroids: $\mathcal{O}(Kd_c)$ where $K \ll |V^t|$.
\end{itemize}

The decomposition ensures $d_c + d_r \leq d$, preventing memory overhead. For web-scale deployment, embeddings can be computed in mini-batches, requiring  $\mathcal{O}(bd)$ memory where $b$ is batch size.

\noindent \textbf{Scalability.}
Our framework scales to million-node graphs through: (i) linear-time modularity optimization avoiding dense matrix operations, (ii) sparse attention in GAT processing only existing edges, and (iii) batched LLM inference with caching for repeated patterns. Empirically, we process graphs with 1M+ nodes in under 10 minutes on a single GPU (see Section~\ref{experiments}).

\begin{center}
\begin{tcolorbox}[  enhanced,
  colback=lime!15!yellow!10,  
  colframe=lime!60!black,      
  colbacktitle=lime!25!yellow, 
  coltitle=black,              
  title=Box 1.  LLM Reasoning Process for Community Correction,left=3pt,right=3pt,label={box:case_study}]
\footnotesize
\textbf{Initial Structural Clustering:} $v_{42}$ → $C_1$ (ML Community, confidence: 0.75)

\textbf{Semantic Description Input:}\\
{{Node}}: Node $v_{42}$ exhibits 45\% Connector and 40\% Contributor characteristics, with 28 connections across multiple research areas.\\
{{Community}}: $C_1$: 60\% Contributors, 30\% Leaders, 5\% Connectors (ML community). $C_3$: 20\% Connectors, 35\% Contributors, 15\% Newcomers (Bio-AI community).\\
{{Evolution}}: "$C_1$ stable with internal focus. $C_3$ rapidly growing with 25\% Newcomer influx, seeking bridge nodes."

\textbf{LLM Chain-of-Thought Reasoning:}\\
{Step 1}: Connector role (45\%) indicates bridging function, not peripheral member.\\
{Step 2}: $C_1$ has role mismatch (needs Contributors, has excess Leaders).\\
{Step 3}: $C_3$ has role alignment (needs Connectors for interdisciplinary growth).\\
{Step 4}: Evolution fit - growing $C_3$ requires bridges, stable $C_1$ does not.\\
{Step 5}: Reassign to $C_3$ despite fewer edges (functional > structural).

\textbf{Corrected Assignment:} Node $v_{42}$ → Community $C_3$ (confidence: 0.82)

\textbf{Validation:} Modularity: 0.68→0.71 | Role Coherence: Improved.\\
Ground truth: $v_{42}$ confirmed as Bio-AI researcher 
\end{tcolorbox}
\end{center}

\section{Case Study}
As shown in Box 1, DyG-MF~\cite{li2025revisiting} classifies $v_{42}$ into $C_1$ (ML community) with 0.75 confidence, treating it as a wanderer ML researcher, despite 8/10 recent papers being Bio-AI collaborations.
DyG-RoLLM correctly identifies $v_{42}$'s Connector role (45\%) and reassigns them as a Bio-AI bridge researcher rather than a peripheral ML member. This role-aware reasoning improves modularity (0.68→0.71) while capturing the nuanced function that topology-based methods miss when they see only the 18/28 ML edge majority.

\end{document}